%% file: iclr2022_conference.tex
\documentclass{article} %
\usepackage{iclr2022_conference,times}

\input{math_commands.tex}

\usepackage{booktabs}
\usepackage{hyperref}
\usepackage{url}
\usepackage{graphicx}

\title{A Unified Framework for Multi-distribution Density Ratio Estimation}

\iclrfinalcopy
\author{Lantao Yu\\
Department of Computer Science\\
Stanford University\\
\texttt{lantaoyu@cs.stanford.edu}
\And
Yujia Jin\\
Department of Management Science and Engineering\\
Stanford University\\
\texttt{yujiajin@stanford.edu}
\And 
Stefano Ermon\\
Department of Computer Science\\
Stanford University\\
\texttt{ermon@cs.stanford.edu}
}

\begin{document}

\maketitle

\begin{abstract}
Binary density ratio estimation (DRE), the problem of estimating the ratio $p_1/p_2$ given their empirical samples, provides the foundation for many state-of-the-art machine learning algorithms such as contrastive representation learning and covariate shift adaptation. In this work, we consider a generalized setting where given samples from multiple distributions $p_1, \ldots, p_k$ (for $k > 2$), we aim to efficiently estimate the density ratios between all pairs of distributions. Such a generalization leads to important new applications such as estimating statistical discrepancy among multiple random variables like multi-distribution $f$-divergence, and bias correction via multiple importance sampling. We then develop a general framework from the perspective of Bregman divergence minimization, where each strictly convex multivariate function induces a proper loss for multi-distribution DRE. Moreover, we rederive the theoretical connection between multi-distribution density ratio estimation and class probability estimation, justifying the use of any strictly proper scoring rule composite with a link function for multi-distribution DRE. We show that our framework leads to methods that strictly generalize their counterparts in binary DRE, as well as new methods that show comparable or superior performance on various downstream tasks.
\end{abstract}

\section{Introduction}
Estimating the density ratio between two distributions based on their empirical samples is a central problem in machine learning, which continuously drives progress in this field and finds its applications in many machine learning tasks such as anomaly detection \citep{hido2008inlier,smola2009relative,hido2011statistical}, importance weighting in covariate shift adaptation \citep{huang2006correcting,sugiyama2007direct}, generative modeling \citep{uehara2016generative,nowozin2016f,grover2019bias}, two-sample test \citep{sugiyama2011least,gretton2012kernel}, mutual information estimation and representation learning \citep{oord2018representation,hjelm2018learning}.
It is such a powerful paradigm because computing density ratio focuses on extracting and preserving contrastive information between two distributions, which is crucial in many tasks.
Despite the tremendous success of binary DRE, many applications involve more than two probability distributions and developing density ratio estimation methods among multiple distributions has the potential of advancing various applications such as estimating multi-distribution statistical discrepancy measures \citep{garcia2012divergences},
multi-domain transfer learning, bias correction and variance reduction with multiple importance sampling \citep{elvira2019generalized}, multi-marginal generative modeling \citep{cao2019multi} and multilingual machine translation \citep{dong2015multi,aharoni2019massively}.

Although recent years have witnessed significant progress and a continuously increasing trend in developing more sophisticated and advanced methods for binary DRE \citep{sugiyama2012density,liu2017trimmed,rhodes2020telescoping,kato2021non,choi2021featurized}, methods for estimating density ratios among multiple distributions remain largely unexplored, besides an empirical exploration of multi-class logistic regression for multi-task learning \citep{bickel2008multi}, where the density ratios serve as the resampling weights between the distribution of a pool of examples of multiple tasks and the target distribution for a given task at hand and lead to significant accuracy improvement on HIV therapy screening experiments.

In this work, we propose a unified framework based on expected Bregman divergence minimization, where any strictly convex multivariate function induces a proper loss for multi-distribution DRE, thus generalizing the framework in \citep{sugiyama2012density} to multi-distribution case.
Moreover, by directly generalizing the Bregman identity in \citep{menon2016linking} to multi-variable case, we rederive a similar result to \citep{nock2016scaled}, which formally relates losses for multi-distribution density ratio estimation and class probability estimation and theoretically justifies the use of any strictly proper scoring rule (e.g., the logarithm score \citep{good1952rational}, the Brier score \citep{brier1950verification} and the pseudo-spherical score \citep{good1971comment}) composite with a link function for multi-distribution DRE. 
By choosing a variety of specific convex functions or proper scoring rules, we show that our unified framework leads to methods that strictly generalize their counterparts for binary DRE, as well as new objectives specific to multi-distribution DRE. We demonstrate the effectiveness of our framework, and study and compare the empirical performance of its different instantiations on various downstream tasks that rely on accurate multi-distribution density ratio estimation.

\section{Preliminaries}\label{sec:preliminary}
\subsection{Multi-class Experiments}\label{sec:multi-class-exp}
In multi-class experiments, we have a pair of random variables $(X, Y) \in \gX \times \gY$ with joint distribution $D(X, Y)$, where $\gX$ is the sample space and $\gY = [k] \defeq \{1, \ldots, k\}$ is the finite label space. 
Define the probability simplex as $\Delta_k \defeq \{\vp \in \sR^k_{\ge0} | \mathbf{1}^\top \vp = 1\}$. 
According to chain rule of probability, any joint distribution $D(X,Y)$ can be decomposed into class priors $\pi_i \defeq \sP(Y=i)$ and class conditionals $P_i(x) \defeq \sP(X=x|Y=i)$ for $i \in [k]$, or into sample marginal $M(x) \defeq \sP(X=x)$ and class probability function $\veta: \gX \to \Delta_k$ (i.e., $\eta_i(x) = \sP(Y=i|X=x)$). We write  $\veta(x)$ as a vector $\veta$ and omit $x$ when it is clear from context. Thus we can also represent the joint distribution as $D = (\bm{\pi}, P_1, \ldots, P_k)$ (where $\vpi \in \Delta_k$) or $(M, \veta)$. For any $i \in [k]$, we assume $P_i$ has density $p_i$ with respect to the Lebesgue measure.

\textbf{Remark on notations.} To avoid confusion, we would like to emphasize that the class probability is denoted as $\eta_i(x) = \sP(Y=i|X=x)$ and the class conditional is denoted as $P_i(x) = \sP(X=x|Y=i)$ with density $p_i(x)$. The former further satisfies the normalization constraint: $\forall x \in \gX, \sum_{i=1}^k \eta_i(x) = 1$, while $i$ in the latter one only serves as the index for $k$ different distributions.

In multi-class classification, given independent and identically distributed (i.i.d.) samples from the joint distribution $D(X,Y)$, we want to learn a probabilistic classifier $\heta: \gX \to \Delta_k$ to approximate the true class probability function $\veta$ by minimizing the following $\ell$-risk: 
\begin{equation}
    \gL_\text{CPE}(\heta; D) = \mathbb{E}_{D(x,y)} [\ell(y, \heta(x))] = \mathbb{E}_{x \sim M}[\mathbb{E}_{y \sim \veta(x)} [\ell(y, \heta(x))]] = \mathbb{E}_{x \sim M} [L(\veta(x), \heta(x))]
    \label{eq:l-risk}
\end{equation}
where $\ell: [k] \times \Delta_k \to \sR$ is the \emph{loss} function for using the class predictor $\heta(x)$ when the true class is $y$, and $L: \Delta_k \times \Delta_k \to \sR$ is the \emph{expected loss} of $\heta(x)$ under the true class probability $\veta(x)$.

\begin{definition}[Proper loss]
A loss function $\ell$ is proper if the corresponding expected loss satisfies: $\forall P, Q \in \Delta_k, L(P,Q) \geq L(P,P)$. It is strictly proper if the equality holds only when $P=Q$.
\end{definition}
In statistical decision theory \citep{gneiting2007strictly}, the negative proper loss is also called \emph{proper scoring rule} (i.e., $S(y, \hat{\veta}(x)) = - \ell(y, \hat{\veta}(x))$), which assesses the utility of the prediction.
Properness of a loss is desirable in multi-class classification because it encourages the class probability estimator $\heta$ to match the true class probability function $\veta$. An important property of proper loss is summarized in the following theorem:

\begin{definition}[Bregman divergence]\label{def:bregman-divergence}
Given a differentiable convex function $\phi: \gS \to \mathbb{R}$ defined on a convex set $\gS \subset \mathbb{R}^d$ and two points $\vx,\vy \in \gS$, the Bregman divergence from $\vx$ to $\vy$ is defined as:
\begin{equation}
    \breg_\phi(\vx,\vy) \defeq \phi(\vx) - \phi(\vy) - \langle \vx - \vy, \nabla \phi(\vy) \rangle
\end{equation}
\end{definition}

\begin{restatable}[\citep{gneiting2007strictly}; Proposition 7 in \citep{vernet2011composite}]{theorem}{classificationBregman}\label{thm:regret-bregman}
Given a proper loss $\ell$ and the corresponding expected loss $L$, for any $P, Q \in \Delta_k$, the generalized entropy function $\uL(P) \defeq \inf_{Q \in \Delta_k} L(P, Q) = L(P,P)$ is concave; when $\uL$ is differentiable, the regret or excess risk of a predictor $Q$ over the Bayes-optimal $P$ is the Bregman divergence induced by the convex function $f=-\uL$:
\begin{equation}
    \reg(P, Q; \ell) \defeq L(P,Q) - L(P,P) = \breg_f(P, Q)
\end{equation}
\end{restatable}
Given the Bregman divergence representation of the point-wise regret in Theorem~\ref{thm:regret-bregman} and the $\ell$-risk in Equation~(\ref{eq:l-risk}), the excess risk of a class probability estimator $\heta$ over the Bayes optimal $\veta$ is:
\begin{equation}\label{eq:breg-represent-regret}
\begin{aligned}
    \reg(\heta; M, \veta, \ell) \defeq & \gL_\text{CPE}(\heta; D) - \gL_\text{CPE}(\veta; D) = \mathbb{E}_{M(x)} [L(\veta(x), \heta(x)) - L(\veta(x), \veta(x))] \\
    =& \mathbb{E}_{M(x)} [\breg_f(\veta(x), \heta(x))]
\end{aligned}
\end{equation}

\subsection{Multi-distribution $f$-Divergence}

Csisza\'r's $f$-divergence is a popular way to measure the discrepancy between two probability distributions. Specifically, given two distributions $P, Q$ and a convex function $f:\R_{+}\to \R\cup\{\pm\infty\}$ satisfying $f(1) = 0$, the $f$-divergence between $P$ and $Q$ is defined as $\divg_f(P||Q) = \mathbb{E}_{Q} [f(\mathrm{d} P / \mathrm{d} Q)]$. In the following, we will introduce the multi-distribution extension of $f$-divergence \citep{garcia2012divergences}.

\begin{definition}[Multi-distribution $f$-divergence]
For $k$ probability distributions $P_1,\ldots,P_k$ on a common probability space $(\gX,\sigma(\gX))$ with densities $p_1, \ldots, p_k$, given multi-variate closed convex function $f:\R_{+}^{k-1}\to \R\cup\{\pm\infty\}$ satisfying $f(\1)=0$, the multi-distribution $f$-divergence between $P_1,\ldots,P_{k-1}$ and $P_k$ is defined as:
\begin{equation}
    \divg_f\left(P_1,\ldots,P_{k-1}||P_k\right) = \E_{p_k(x)} \left[f\left(\frac{p_1(x)}{p_k(x)},\ldots, \frac{p_{k-1}(x)}{p_k(x)}\right)\right]
    \label{eq:multi-f-divergence}
\end{equation}
\end{definition}

\subsection{Connecting Density Ratios and Class Probabilities via Link Function}\label{sec:multi-class-link-function}
Inspired by the definition in \eqref{eq:multi-f-divergence}, we consider the following canonical density ratio vector (more discussion about this choice can be found in Section~\ref{sec:multi-class-dre-method}): $\vr(x) = (r_1(x), \ldots, r_{k}(x))$ where $r_i(x) \defeq p_i(x) / p_k(x)$ and $r_k(x)=1$. Then we can connect a density ratio vector $\vr(x) \in \R_{+}^{k-1} \times \{1\}$ and a class probability vector $\veta(x) \in \Delta_k$ via an invertible link function.

According to Bayes' theorem, we have:
\begin{equation}\label{eq:bayes-rule}
\frac{\sP(X=x, Y=i)}{\sP(X=x, Y=k)} =\frac{\pi_i p_i(x)}{\pi_k p_k(x)} = \frac{M(x)\eta_i(x)}{M(x)\eta_k(x)}\Leftrightarrow  r_i(x) = \frac{p_i(x)}{p_k(x)} =  \frac{\pi_k}{\pi_i}\cdot\frac{\eta_i(x)}{\eta_k(x)}.
\end{equation}
Thus we define the following multi-distribution link function $\Psi_{\mathrm{dr}}:\Delta^k\rightarrow \R_{+}^{k-1} \times \{1\}$ as a natural generalization of the binary DRE link function \citep{menon2016linking,vernet2011composite}:
\begin{equation}\label{eq:DRE-link} 
[\Psi_{\mathrm{dr}}(\veta(x))]_i \defeq \frac{\pi_k}{\pi_i}\cdot \frac{\eta_i(x)}{\eta_k(x)} = r_i(x),~\text{for all}~i\in[k].
\end{equation}
Given \eqref{eq:DRE-link} and the normalization constraint $\sum_{i \in [k]} \eta_i = 1$, we obtain the inverse link function:
\begin{equation}\label{eq:DRE-link-inverse} 
[\Psi^{-1}_{\mathrm{dr}}(\vr(x))]_i \defeq \frac{\pi_i r_i(x)}{\sum_{j\in[k]}\pi_j r_j(x)}=\eta_i(x),~\text{for all}~i\in[k].
\end{equation}
Thus given knowledge of the prior distribution $\vpi$ (which can also be easily estimated from empirical samples), one can transform a class probability estimator into a density ratio estimator via $\hat{\vr}(x) = \Psi_{\mathrm{dr}}(\heta(x))$ and vice versa via $\heta(x) = \Psi^{-1}_{\mathrm{dr}}(\hat{\vr}(x))$.

\section{A Unified Framework for Multi-distribution DRE}
\subsection{Multi-distribution Density Ratio Estimation Problem Setup}\label{sec:multi-class-dre-setup}
Following the basic formulation of multi-class experiments in Section~\ref{sec:multi-class-exp}, we now introduce the problem setup of multi-distribution density ratio estimation (DRE). Recall that $\gX$ is the common data domain and $P_1, \ldots, P_k$ are $k$ different distributions defined on $\gX$ with densities $p_1, \ldots, p_k$. Suppose we are given $n_i$ i.i.d. samples $\{x_j^{(i)}\}_{j=1}^{n_i}$ from each distribution $P_i$. The goal of multi-distribution DRE is to estimate the density ratios between all pairs of distributions $\{r_{ij} \defeq p_i/p_j\}_{i,j \in [k]}$ from the i.i.d. datasets $\{\{x_j^{(i)}\}_{j=1}^{n_i}\}_{i=1}^k$. 
In this paper, we assume that the density ratios are always well-defined on domain $\gX$ (e.g., when the distributions have strictly positive densities), which is also a common assumption in binary DRE problem \citep{kanamori2009least,kato2021non}.

A naive approach towards this problem is to separately estimate each density $p_i$ from $\{x_j^{(i)}\}_{j=1}^{n_i}$ and then plug in $p_i$ and $p_j$ to get $r_{ij}$. However, as previous theoretical works \citep{kpotufe2017lipschitz,nguyen2007estimating,kanamori2012statistical,que2013inverse} suggest, directly estimating density ratios has many advantages in practical settings. Specifically, we know that (1) optimal convergence rates depend only on the smoothness of the density ratio and not on the densities; (2) optimal rates depend only on the intrinsic dimension of data, thus escaping the curse of dimension in density estimation. Inspired by these observations in binary DRE, this paper aims to develop a general framework for directly estimating multi-distribution density ratios. Moreover, we also theoretically prove that various interesting facts \citep{menon2016linking,sugiyama2012density}, which hold in the binary case, extend to our multi-distribution case in Section~\ref{sec:theory}.

While most previous works focus on DRE in binary cases, multi-distribution DRE has many important downstream applications.
For example, given any integrable function $\phi: \gX \to \sR$, suppose we want to use importance sampling to estimate the expectation of $\phi$ with respect to a target distribution $Q$ with density $q$ w.r.t. the base measure:
\begin{equation}
    \mathbb{E}_{q(x)} [\phi(x)] = \int_\gX q(x) \phi(x) \mathrm{d}x = \int_\gX p(x) \frac{q(x)}{p(x)} \phi(x) \mathrm{d}x = \mathbb{E}_{p(x)}\left[r(x) \cdot \phi(x)\right]
\end{equation}
where we use the density ratio $r = p/q$ to correct the bias caused by using samples from the proposal distribution $p$ rather than the target distribution $q$.
However, in practice, finding a good proposal is critical yet challenging \citep{owen2000safe}. An alternative and more robust strategy is to use a population of different proposals (sampling schemes) and use a set of density ratios to correct the bias, which is also known as multiple importance sampling (MIS) \citep{cappe2004population,elvira2015efficient}. Given $k$ different proposals $p_1, \ldots, p_k$, the MIS estimation of the expectation is given by:
\begin{equation}
    \mathbb{E}_{q(x)} [\phi(x)] = \sum_{i=1}^k \omega_i \mathbb{E}_{p_i(x)}\left[\frac{q(x)}{p_i(x)}\phi(x)\right]
\end{equation}
where $\omega_i$ is the weight for each proposal $p_i$ and satisfies $\sum_i \omega_i = 1$. Thus a more efficient and accurate multi-distribution DRE method will lead to better MIS. In the context of multi-source off-policy policy evaluation \citep{kallus2021optimal}, the proposals correspond to a set of demonstration policies and the target distribution is the query policy whose performance we want to evaluate from the offline multi-souce demonstrations; in the context of multi-domain transfer learning setting (covariate shift adaptation) \citep{bickel2008multi,dinh2013fidos}, the proposals correspond to a set of data generating distributions (e.g. multiple source domains or various data augmentation strategies) and the target is the test distribution we care about. Estimating multi-distribution density ratios also allows us to compute important information quantities among multiple random variables such as 
the multi-distribution $f$-divergence in Equation~(\ref{eq:multi-f-divergence}), which can be used to analyze various kinds of discrepancy and correlations between multiple random variables and further has the potential of inspiring new generative models for multiple marginal matching problem \citep{cao2019multi}.

\subsection{Multi-distribution DRE via Bregman Divergence Minimization}\label{sec:multi-class-dre-method}
Inspired by the success of Bregman divergence minimization for unifying various DRE methods in the binary case \citep{sugiyama2012density}, in this section, we propose a general framework for solving the multi-distribution density ratio estimation problem. First, we discuss our modeling choices. Although our goal is to estimate ${k \choose 2}$ density ratios (between all possible pairs), the solution set $\{r_{ij} \defeq p_i/p_j\}_{i,j \in [k]}$ actually has $k-1$ degrees of freedom (e.g., $r_{ik}=r_{ij} \cdot r_{jk}$). Thus without loss of generality, we parametrize the following $k-1$ density ratio models $\hat{\vr}_\vtheta = (\hat{r}_{\theta_1}, \ldots, \hat{r}_{\theta_{k-1}})$ to approximate the true canonical density ratios $\vr = (r_1, \ldots, r_{k-1})$, where $r_i \defeq p_i / p_k$ for $i \in [k-1]$. For the simplicity of notation, we will omit the dependence on the parameters $\vtheta$ and write our density ratio models as $\hat{\vr} = (\hat{r}_1, \ldots, \hat{r}_{k-1})$. An advantage of such modeling choice is that any density ratio can be recovered within one step of computation $\frac{p_i}{p_j} = \frac{p_i / p_k}{p_j / p_k} = \frac{r_i}{r_j}$, thus avoiding large compounding error while naturally ensuring consistency within the solution set (i.e., if we parametrize $\hat{r}_{ij}$, $\hat{r}_{jk}$ and $\hat{r}_{ik}$ respectively, we have to make sure they satisfy $\hat{r}_{ik}=\hat{r}_{ij} \cdot \hat{r}_{jk}$).

Since our goal is to optimize $\hat{\vr}$ to approximate the true density ratios $\vr$, we consider to use Bregman divergence (Def.~\ref{def:bregman-divergence}) to measure the discrepancy between $\vr$ and $\hat{\vr}$.
Specifically, for any strictly convex function $f: \mathbb{R}^{k-1}_+ \to \mathbb{R}$ and $\forall x \in \gX$, we have the following point-wise optimization problem:
\begin{equation}
    \min_{\hat{\vr}(x) \in \sR^{k-1}_+} \breg_f(\vr(x),\hat{\vr}(x)) = f(\vr(x)) - f(\hat{\vr}(x)) - \langle \nabla f(\hat{\vr}(x)), \vr(x) - \hat{\vr}(x) \rangle
\end{equation}
which corresponds to the difference between the value of $f$ at $\vr$, and the value of the first-order Taylor expansion of $f$ around point $\hat{\vr}$ evaluated at point $\vr$. Although the current formulation can be understood as a regression
problem from $\hat{\vr}(x)$ to the true density ratios $\vr(x)$, we actually only have i.i.d. samples $x \sim p_1, \ldots, p_k$ instead of the true targets $\vr(x)$. In this case, we consider to use the following expected Bregman divergence to measure the overall discrepancy from the true density ratios $\vr$ to the density ratio models $\hat{\vr}$:
\begin{align}
    \gL_\text{DRE}(\hat{\vr}; D) =& \int_\gX p_k(x) \Big(f(\vr(x)) - f(\hat{\vr}(x)) - \langle \nabla f(\hat{\vr}(x)), \vr(x) - \hat{\vr}(x) \rangle\Big) \mathrm{d}x \label{eq:multi-class-dre-orig}\\
    =& \mathbb{E}_{p_k(x)}\left[\langle \nabla f(\hat{\vr}(x)), \hat{\vr}(x) \rangle - f(\hat{\vr}(x))\right] - \sum_{i \in [k-1]} \mathbb{E}_{p_i(x)} [\partial_i f(\hat{\vr}(x))] + C
\end{align}
where $C \defeq \int_\gX p_k(x) f(\vr(x)) \mathrm{d}x = \divg_f(P_1,\ldots,P_{k-1}\|P_k)$ is a constant with respect to $\hat{\vr}$ and the equality comes from the fact that $p_k \cdot (r_1, \ldots, r_{k-1}) = (p_1, \ldots, p_{k-1})$ according to the definition of $\vr$. The rationale behind the above choice is that it allows us to get an unbiased estimation of the discrepancy between $\vr$ and $\hat{\vr}$ only using i.i.d. samples from $p_1, \ldots, p_k$. Specifically, since $C$ is a constant, we have the following optimization problem over $\hat{\vr}$ to approximate the true density ratios (where each expectation $\mathbb{E}_{p_i}$ can be empirically estimated using samples from $p_i$):
\begin{equation}
    \min_{\hat{\vr}:\gX \to \sR^{k-1}_+}
    \mathbb{E}_{p_k(x)}\left[\left\langle \nabla f(\hat{\vr}(x)), \hat{\vr}(x) \right\rangle - f(\hat{\vr}(x))\right] - \sum_{i \in [k-1]} \mathbb{E}_{p_i(x)} \left[\partial_i f(\hat{\vr}(x))\right]
    \label{eq:multi-class-dre}
\end{equation}

Interestingly, the above multi-distribution DRE formulation, which is based on Bregman divergence minimization, can be alternatively derived from the perspective of variational estimation of multi-distribution $f$-divergence. In the following, We briefly discuss such an interpretation of Eq.~(\ref{eq:multi-class-dre}).

Based on Fenchel duality, we can represent any strictly convex function $f:\R_{+}^{k-1}\to \R\cup\{+\infty\}$ through its conjugate function $f^*(\vs) \defeq \max_{\vr\in\sR^{k-1}_{+}} \langle \vs, \vr \rangle - f(\vr)$ as:
\begin{equation}
    f(\vr(x)) = \max_{\vs: \gX \to \sR^{k-1}} \langle \vr(x), \vs(x)\rangle - f^*(\vs(x)),~~\text{for any}~x\in\gX.\label{eq:fenchel-duality}
\end{equation}
In order to estimate the multi-distribution $f$-divergence defined in Eq.~(\ref{eq:multi-f-divergence}) only using samples from $P_1, \ldots, P_k$ (instead of their density information), we consider the following variational representation of multi-distribution $f$-divergence by substituting Eq.~(\ref{eq:fenchel-duality}) into Eq.~(\ref{eq:multi-f-divergence}):
\begin{align}
    \divg_f(P_1,\ldots, P_{k-1}||P_k)= -\min_{\vs:\gX\rightarrow \R^{k-1}}\left[-\sum_{i\in[k-1]}\E_{p_i(x)}[\vs(x)]_i + \E_{p_k(x)}f^*(\vs(x))\right]\label{eq:obj-fdivg}
\end{align}

We then have the following lemma revealing the equivalence between the optimization problem in Eq.~(\ref{eq:multi-class-dre}) and Eq.~(\ref{eq:obj-fdivg}).

\begin{restatable}[DRE via variational estimation of multi-distribution $f$-divergence]{proposition}{restateEquivalence}\label{lem:equivalence-opt}
Given a strictly convex function $f:\sR^{k-1}_+\rightarrow \R\cup\{+\infty\}$, the optimization problem in Eq.~(\ref{eq:multi-class-dre}) (induced by minimizing expected Bregman divergence $\breg_f(\vr, \hat{\vr})$) is equivalent to the one in Eq.~(\ref{eq:obj-fdivg}) (for variational estimation of multi-distribution $f$-divergence) under change of variables satisfying: $\nabla f(\hat{\vr}(x))=\vs(x),~\forall x\in\gX$.
\end{restatable}

\section{Connecting Losses for Multi-class Classification and DRE}\label{sec:theory}

In this section, we rederive a similar result to \citep{nock2016scaled} by directly generalizing the Bregman identity in \citep{menon2016linking} to multi-variable case, which established the theoretical connection between multi-distribution DRE and multi-class classification.

In Section~\ref{sec:multi-class-exp}, we have shown that the exact minimization of the excess risk for any strictly proper loss $\ell$ results in the true class probability function $\veta$, and consequently gives us the true density ratio $\vr$ through the link function $\Psi_{\mathrm{dr}}(\veta)$. 
In the following, we take a further step to show that essentially the procedure of minimizing any strictly proper loss is equivalent to minimizing an expected Bregman divergence between the true density ratios $\vr$ and the approximate density ratios $\hat{\vr}$, thus generalizing the theoretical results in binary case \citep{menon2016linking} to the multi-distribution case and justifying the validity of using any strictly proper scoring rule (e.g. Brier score \citep{brier1950verification} and pseudo-spherical score \citep{good1971comment}) for multi-distribution DRE. 
All proofs for this section can be found in Appendix~\ref{app:proof-sec-4}.

We start by introducing the following multivariate Bregman identity.
\begin{restatable}[Multivariate Bregman Identity]{lemma}{bregidentity}\label{lem:breg-identity}
Given a convex function $f:\R^{k-1}\to \R$,
we can define an associated function $\fopt(u_1,\ldots,u_{k-1}) = (1+\sum_{i\in[k-1]}u_i)f\left(\frac{1}{1+\sum_{i\in[k-1]}u_i}\cdot \vu\right)$. We can show that (i) $\fopt$ is convex and (ii) for any $\vu,\vv\in\R^{k-1}$, their associated Bregman divergences satisfy:
\begin{equation}\label{eq:bregman-identity-basic}
\breg_f\left(\frac{1}{1+\sum_{i\in[k-1]}u_i}\cdot \vu,\frac{1}{1+\sum_{i\in[k-1]}v_i}\cdot \vv\right) = \frac{1}{1+\sum_{i\in[k-1]}u_i}\breg_{\fopt}(\vu,\vv).
\end{equation}
\end{restatable}
One can then apply Lemma~\ref{lem:breg-identity} with $u_i = \frac{\pi_i}{\pi_k} r_i$ and $v_i = \frac{\pi_i}{\pi_k}\hr_i$ for each $i\in[k-1]$ and use the fact that $\breg_{\fopt_\pi}\left(\vr,\hat{\vr}\right) = \breg_{\fopt}(\vu, \vv)$ for $\fopt_\pi(\vr) = \fopt(\frac{1}{\pi_k}\vpi\circ \vr)$ to establish the following connection between the optimality gap of density ratio estimators and class probability estimators, where we use $\va \circ \vb$ to denote the element-wise product between vectors $\va$ and $\vb$, and  $\vpi_{[1:k-1]}\in\R^{k-1}$ as the vector when restricting $\vpi$ onto its first $k-1$ coordinates.

\begin{restatable}{proposition}{corBregIdentity}\label{cor:breg-identity}
For any convex function $f:\R^{k-1}_{+}\rightarrow \R$, and two density ratio vectors $\vr(x)$ and $\hat{\vr}(x)$, one can construct corresponding class probability vectors
$\veta(x) = \Psi_\mathrm{dr}^{-1}(\vr(x))$ and $\hat{\veta}(x) = \Psi_\mathrm{dr}^{-1}(\hat{\vr}(x))$ through the inverse link function in \eqref{eq:DRE-link-inverse}, and obtain:
\begin{equation}\label{eq:bregman-identity-applied}
    \breg_f\left(\veta(x), \hat{\veta}(x)\right) = \frac{\pi_k}{\pi_k+\sum_{i\in[k-1]}\pi_ir_i(x)}\breg_{\fopt_\pi}\left(\vr(x),\hat{\vr}(x)\right)~\text{for all}~x\in\gX,
\end{equation}
where we define the convex function $\fopt_\pi$ induced by some prior distribution $\pi \in\Delta_k $ as 
\begin{equation}\label{eq:fopt-pi-def}
    \fopt_\pi(r_1,\ldots, r_{k-1}) \defeq \left(1+\sum_{i\in[k-1]}\pi_i r_i/\pi_k\right)\cdot f\left(\frac{\vpi_{[1:k-1]} \circ\vr}{\pi_k+\sum_{i\in[k-1]}\pi_i r_i} \right).
\end{equation}
\end{restatable}
Combining Proposition~\ref{cor:breg-identity} with the Bregman divergence representation of the point-wise regret for a proper risk $\ell$ for multi-class classification in~\eqref{eq:breg-represent-regret}, we provide the following main theorem that interprets the minimization of multi-class classification regret as multi-distribution DRE under expected Bregman divergence minimization.

\begin{restatable}{theorem}{thmBregmanDRE}\label{thm:bregman-dre}
Given any strictly proper loss $\ell$, for any joint data distribution $D(X,Y)$ with class prior $\pi\in\Delta_k$, the multi-class classification regret defined in~\eqref{eq:breg-represent-regret} satisfies that:
\begin{equation}
    \reg(\heta; M, \veta, \ell) = \pi_k \E_{ p_k(x)}\breg_{\fopt_\pi}(\vr(x),\hat{\vr}(x)),
\end{equation}
for $\fopt_\pi$ as defined in~\eqref{eq:fopt-pi-def}, and $\vr = \Psi_\mathrm{dr}(\veta)$ and $\hat{\vr} = \Psi_\mathrm{dr}(\hat{\veta})$ as defined in~\eqref{eq:DRE-link}.
\end{restatable}

Theorem~\ref{thm:bregman-dre} generalizes a known equivalence between density ratio estimation and class probability estimation in the binary case (see Section 5 in \citep{menon2016linking}), and provides a similar equivalence in the more complicated multi-class experiments. Besides, in comparison to the binary case result, we also provide a simpler proof, loosen the assumptions on the twice-differentiability of convex function $f$ induced by the proper loss $\ell$ (i.e., $f = -\underline{L}$, see Theorem~\ref{thm:regret-bregman} for more details), and generalize the argument to an arbitrary prior distribution $\pi\in\Delta^k$ instead of the uniform prior case $\pi_1=\pi_2=1/2$ considered in~\citep{menon2016linking}. 

Moreover, we notice that multi-distribution $f$-divergence among class conditionals $P_1, \ldots, P_k$ also corresponds to the statistical information measure in multi-class experiments \citep{degroot1962uncertainty} (defined as the gap between the prior and posterior generalized entropy). Since we have established the equivalence between multi-class DRE (\eqref{eq:multi-class-dre}) and variational estimation of multi-distribution $f$-divergence (\eqref{eq:obj-fdivg}), we can show by choosing particular convex functions (associated with the loss $\ell$ for multi-class classification), multi-distribution DRE can be viewed as estimating the statistical information measure in multi-class experiments. See detailed discussions in Appendix~\ref{app:proof-sec-4-info}.

\section{Examples of Multi-distribution DRE}\label{sec:examples}
In the binary density ratio matching under Bregman divergence framework \citep{sugiyama2012density}, we can choose various convex functions to recover popular binary DRE methods such as KLIEP \citep{sugiyama2008direct}, LSIF \citep{kanamori2009least} and Logistic Regression \citep{franklin2005elements}.
In this section, we provide some instantiations of our multi-distribution DRE framework. Specifically, Section~\ref{sec:multi-class-dre-method} suggests that any strictly convex multivariate function $f: \sR^{k-1}_+ \to \sR$ induces a proper loss for multi-distribution DRE, and Section~\ref{sec:theory} justifies that any strictly proper scoring rule composite with $\Psi_\mathrm{dr}$ can also be used for multi-distribution DRE.
We briefly discuss some choices of the convex function or proper scoring rule, and we provide detailed derivations in Appendix~\ref{app:examples-derivations}.

\subsection{Methods Induced by Convex Functions}\label{sec:examples-convex}
\textbf{Multi-class Logistic Regression.}~~From Section~\ref{sec:multi-class-link-function}, we know that there is a one-to-one correspondence between a class probability estimator and a density ratio estimator:
$\hat{\vr} = \Psi_\mathrm{dr} \circ \hat{\veta}$ and $\hat{\veta} = \Psi_\mathrm{dr}^{-1} \circ \hat{\vr}$. For the clarity of presentation, here we assume the class prior distribution $\vpi$ is uniform such that $\hr_i(x) = \hat{\eta}_i(x) / \hat{\eta}_k(x)$ and $\hat{\eta}_i(x) = \hr_i(x) / \sum_{j=1}^k \hr_j(x)$. 
To recover the loss of multi-class logistic regression, we choose the following convex function to be $f(\hr_1, \ldots, \hr_{k-1}) = \frac{1}{k} \sum_{i=1}^{k} \hr_i \log\left(\hr_i / \sum_{j=1}^k \hr_j\right)$.
In this case, the loss in \eqref{eq:multi-class-dre} reduces to:
\begin{equation}
    \frac{1}{k}\mathbb{E}_{p_k(x)} \left[\log\left(\sum_{j=1}^{k} \hr_j(x)\right)\right] - \frac{1}{k}\sum_{i=1}^{k-1} \mathbb{E}_{p_i(x)} \left[\log\left(\frac{\hr_i(x)}{\sum_{j=1}^k \hr_j(x)}\right)\right] = -\left(\frac{1}{k}\sum_{i=1}^k \mathbb{E}_{p_i(x)}[\log \hat{\eta}_i(x)]\right)
\end{equation}
We provide discussions for the general case (non-uniform prior $\vpi$) in Appendix~\ref{app:multi-class-lr-derivations}.
Interestingly, we noticed that the above convex function also gives rise to the multi-distribution Jensen-Shannon divergence \citep{garcia2012divergences} (also known as the information radius \citep{sibson1969information}, $\divg_f(P_1, \ldots, P_k) = \frac{1}{k}\sum_{i=1}^k \KL(P_i\| \frac{1}{k}\sum_{j=1}^k P_j)$) up to a constant of $\log k$.

\textbf{Multi-LSIF.}~~When the convex function associated with the Bregman divergence is chosen to be $f(\hr_1, \ldots, \hr_{k-1}) = \frac{1}{2} \sum_{i=1}^{k-1} (\hr_i - 1)^2 = \frac{1}{2} \|\hat{\vr} - \bm{1}\|^2$,
the loss in \eqref{eq:multi-class-dre} reduces to:
\begin{equation}
    \frac{1}{2} \sum_{i=1}^{k-1} \mathbb{E}_{p_k(x)} \left[\hr_i ^2(x) - 1\right] - \sum_{i=1}^{k-1} \mathbb{E}_{p_i(x)} \left[\hr_i(x) - 1\right]
    = \frac{1}{2} \sum_{i=1}^{k-1} \mathbb{E}_{p_k(x)} \left[(\hr_i(x) - r_i(x))^2\right] - C
\end{equation}
where $C = \mathbb{E}_{p_k(x)} \left[\|\vr(x) - 1\|^2\right]$ is a constant w.r.t. $\hat{\vr}$ and the minimizer to the above loss function matches the true density ratios, which strictly generalizes the Least-Squares Importance Fitting (LSIF) \citep{kanamori2009least} method to the multi-distribution case.

Besides, we also consider the following simple convex functions that either strictly generalize their binary DRE counterparts as above, or lead to completely new methods for multi-distribution DRE:
\begin{itemize}
    \item \textbf{Multi-KLIEP.}~~$f(\hr_1, \ldots, \hr_{k-1}) = \sum_{i=1}^{k-1} (\hr_i \log \hr_i - \hr_i)= \langle \hat{\vr}, \log(\hat{\vr}) \rangle - \|\hat{\vr}\|_1$. This strictly generalizes the Kullback–Leibler Importance Estimation Procedure (KLIEP) \citep{sugiyama2008direct} to the multi-distribution case. See Appendix~\ref{app:kliep} for more details.

    \item \textbf{Power.}~~$f(\hr_1, \ldots, \hr_{k-1}) = \sum_{i=1}^{k-1} \hr_i^{\alpha}= \|\hat{\vr}\|_{\alpha}^{\alpha}$, ~~for $\alpha > 1$. This strictly generalizes the robust DRE method in \citep{sugiyama2012density}, which recovers  Multi-KLIEP when $\alpha \to 1$ and Multi-LSIF when $\alpha = 2$. See Appendix~\ref{app:robust-dre} for more details.
    
    \item \textbf{Quadratic.}~~$f(\hr_1, \ldots, \hr_{k-1}) = \hat{\vr}^\top \mH \hat{\vr} + \vq^\top \hat{\vr}$, for any positive definite matrix $\mH \succ 0$.
    When $\mH$ is the identity matrix and $\vq = (-2, \ldots, -2)$, this is equivalent to Multi-LSIF.
    
    \item \textbf{LogSumExp.}~~$f(\hr_1, \ldots, \hr_{k-1}) = \alpha \log \left(\sum_{i=1}^{k-1} \exp(\hr_i / \alpha)\right)$ for $\alpha > 0$.
\end{itemize}
In principle, we can use any desired strictly convex function $f:\sR^{k-1}_+ \to \sR$ within the optimization problem in \eqref{eq:multi-class-dre}, implying the great potential of our unified framework for discovering novel objectives for multi-distribution DRE.
In terms of modeling flexibility, the curvature of different convex functions encode different inductive biases that may favor various downstream applications and we leave the design of more suitable convex functions for DRE as exciting future avenues.

\subsection{Methods Induced by Proper Scoring Rules Composite with $\Psi_\mathrm{dr}$}
From Section~\ref{sec:theory}, we know that any strictly proper loss $\ell: [k] \times \Delta_k \to \sR$ (or strictly proper scoring rule $S(i, \hat{\veta}) = -\ell(i, \hat{\veta})$) in conjunction with the link function $\Psi_\mathrm{dr}$ also induces valid losses for multi-distribution DRE:
\begin{equation}
    \min_{\hat{\vr}:\gX \to \sR^{k-1}_+} \mathbb{E}_{D(x,y)} [\ell(y, \heta(x))] = \mathbb{E}_{x \sim M, y \sim \veta(x)} [\ell(y, \Psi_\mathrm{dr}^{-1}(\hat{\vr}(x)))]
    \label{eq:proper-scoring-rule-dre}
\end{equation}
In this work, we consider using the following classic proper scoring rules \citep{gneiting2007strictly}, where $\heta$ is parametrized as $\Psi_{\mathrm{dr}}^{-1}(\hat{\vr})$ (i.e. $\hat{\eta}_i = \pi_i \hr_i / \sum_{j = 1}^k \pi_j \hr_j$):
\begin{itemize}
    \item \textbf{Logarithm score.} \citep{good1952rational} The loss is specified as $\ell(i, \heta) = -\log(\hat{\eta}_i)$, which also recovers the loss of multi-class logistic regression in Section~\ref{sec:examples-convex}.
    \item \textbf{Brier score.} \citep{brier1950verification} The loss is specified as
    $\ell(i, \heta) = - 2 \hat{\eta}_i + \sum_{j=1}^k \hat{\eta}_j^2 + 1$.
    \item \textbf{Logarithm pseudo-spherical score.} \citep{good1971comment,fujisawa2008robust} 
    The loss is specified as $\ell(i, \hat{\veta}) = - \log \left(\frac{\hat{\eta}_i^{\alpha - 1}}{(\sum_{j=1}^k \hat{\eta}_j^{\alpha})^{(\alpha - 1)/\alpha}}\right)$ for $\alpha > 1$.
\end{itemize}

\section{Experiments}

In this section, we verify the validity of our framework, as well as study and compare the various instantiations introduced in Section~\ref{sec:examples}, on a variety of tasks that all rely on accurate multi-distribution density ratio estimation. In particular, the tasks we consider include density ratio estimation among multiple multivariate Gaussian distributions, anomaly detection on CIFAR-10 \citep{krizhevsky2009learning}, multi-target MNIST Generation \citep{lecun1998gradient} and multi-distribution off-policy policy evaluation. We discuss the basic problem setups, evaluation metrics and experimental results in the following and we provide more experimental details for each task in Appendix~\ref{app:experiment-details}.

\textbf{Synthetic Data Experiments.}~~We first apply our methods to estimate density ratios among $k=5$ multivariate Gaussian distributions with different mean vectors and identity covariance matrix. We conducted experiments for various data dimensions ranging from 2 to 50. Since Gaussian distributions have tractable densities, we know the ground-truth density ratio functions and we calculate the mean absolute error (MAE) between all $k \choose 2$ true density ratios and the learned ones: 
\begin{align*}
\small
    \mathrm{MAE}(\vr, \hat{\vr}; M(x)) = \frac{2}{k(k-1)}\mathbb{E}_{M(x)}\left[\sum_{1 \leq i < j \leq k} \left|r_{ij}(x) - \hat{r}_{ij}(x)\right|\right]
\end{align*}
where density ratio between $p_i$ and $p_j$ is recovered by $\hat{r}_i / \hat{r}_j$ as discussed in Section~\ref{sec:multi-class-dre-method}. We summarize the results in Table~\ref{tab:gaussian}, from which we can see that multi-class logistic regression and Brier score composite with $\Psi_\mathrm{dr}$ show superior performance in this task.

\textbf{OOD Detection on CIFAR-10.}~~Suppose we have $k$ different distributions $p_1, \ldots, p_k$, where $p_k = \sum_{i \in [k-1]} \alpha_i p_i$, ($\sum_{i \in [k-1]} \alpha_i = 1$ and $\forall i, \alpha_i > 0$). For each distribution $p_i$ ($i \leq k-1$), samples from the mixture distribution $p_k$ contain both inlier samples and outlier samples. The goal of this task is to identify the inlier samples from the pool of mixture samples. In particular, we use the learned density ratio $\hat{r}_i$ as the score function to retrieve the inlier samples of $p_i$, since the true density ratio function $r_i = p_i / \sum_{j \in [k-1]} \alpha_j p_j$ tend to be larger for samples from $p_i$ and smaller for samples from other distributions. In this case, we calculate the average AUROC for each scoring function.

\begin{table}[t]
    \centering
    \caption{Mean absolute error for multi-distribution density ratio estimation among five multivariate Gaussian distributions. Results are averaged across three random seeds.}
    \label{tab:gaussian}
        \resizebox{\textwidth}{!}{\begin{tabular}{c|cccccccc}
        \toprule
        Method & $\mathrm{Dim}=2$ & $\mathrm{Dim}=5$ & $\mathrm{Dim}=10$ & $\mathrm{Dim}=20$ & $\mathrm{Dim}=30$ & $\mathrm{Dim}=40$ & $\mathrm{Dim}=50$ \\ 
        \midrule
        Random Init & $1.724 \pm 0.03$ & $1.723 \pm 0.008$ & $1.728 \pm 0.02$ & $1.765 \pm 0.017$ & $1.749 \pm 0.009$ & $1.753 \pm 0.002$ & $1.768 \pm 0.008$\\
        Multi-LR & $\bm{0.044} \pm 0.003$ & $\bm{0.048} \pm 0.005$ & $\bm{0.061} \pm 0.002$ & $\bm{0.07} \pm 0.001$ & $\bm{0.081} \pm 0.002$ & $\bm{0.089} \pm 0.001$ & $\bm{0.098} \pm 0.002$\\
        Multi-KLIEP & $0.051 \pm 0.002$ & $0.066 \pm 0.002$ & $0.074 \pm 0.0$ & $0.089 \pm 0.002$ & $0.105 \pm 0.005$ & $0.112 \pm 0.004$ & $0.123 \pm 0.003$\\
        Multi-LSIF & $0.073 \pm 0.006$ & $0.097 \pm 0.001$ & $0.109 \pm 0.005$ & $0.124 \pm 0.003$ & $0.144 \pm 0.004$ & $0.141 \pm 0.005$ & $0.158 \pm 0.004$\\
        Power & $0.054 \pm 0.003$ & $0.073 \pm 0.001$ & $0.081 \pm 0.004$ & $0.104 \pm 0.003$ & $0.117 \pm 0.003$ & $0.123 \pm 0.005$ & $0.135 \pm 0.004$\\
        Brier & $\bm{0.042} \pm 0.002$ & $\bm{0.056} \pm 0.003$ & $\bm{0.066} \pm 0.003$ & $\bm{0.081} \pm 0.002$ & $\bm{0.086} \pm 0.002$ & $\bm{0.094} \pm 0.002$ & $\bm{0.105} \pm 0.001$\\
        Spherical & $0.103 \pm 0.007$ & $0.106 \pm 0.006$ & $0.115 \pm 0.004$ & $0.121 \pm 0.005$ & $0.125 \pm 0.006$ & $0.132 \pm 0.003$ & $0.138 \pm 0.011$\\
        LogSumExp & $0.231 \pm 0.067$ & $0.198 \pm 0.034$ & $0.184 \pm 0.013$ & $0.179 \pm 0.014$ & $0.184 \pm 0.009$ & $0.192 \pm 0.01$ & $0.193 \pm 0.003$\\
        Quadratic & $0.148 \pm 0.033$ & $0.186 \pm 0.028$ & $0.218 \pm 0.011$ & $0.219 \pm 0.018$ & $0.226 \pm 0.018$ & $0.236 \pm 0.023$ & $0.254 \pm 0.014$\\
        \bottomrule
        \end{tabular}}
        \vspace{-5pt}
\end{table}

\begin{table}[t]
    \centering
    \caption{Results for CIFAR-10 OOD detection, MNIST multi-target generation and multi-distribution off-policy policy evaluation error based on learned density ratios. $\uparrow$ means higher is better and $\downarrow$ means lower is better. Results of top $3$ methods for each task are bold. Results are averaged across three random seeds.}
    \label{tab:cifar-ood}
        \resizebox{0.86\textwidth}{!}{\begin{tabular}{c|ccc}
        \toprule
        Method & CIFAR-10 OOD ($\uparrow$) & MNIST Generation ($\downarrow$) & Off-policy Evaluation ($\downarrow$) \\ 
        \midrule
        Random Init & $0.499 \pm 0.017$ & $1.598 \pm 0.063$ & $1377.68 \pm 379.76$\\
        Multi-LR & $\bm{0.854} \pm 0.009$ & $\bm{0.156} \pm 0.014$ & $62.43 \pm 12.87$ \\
        Multi-KLIEP & $0.828 \pm 0.005$ & $0.281 \pm 0.050$ & $110.89 \pm 35.33$ \\
        Multi-LSIF & $0.801 \pm 0.008$ & $0.274 \pm 0.027$ & $71.09 \pm 1.12$\\
        Power ($\alpha=1.5$) & $0.816 \pm 0.007$ & $0.224 \pm 0.036$ & $\bm{53.43} \pm 20.73$\\
        Brier & $\bm{0.849} \pm 0.010$ & $\bm{0.107} \pm 0.022$ & $71.21 \pm 17.34$\\
        Spherical ($\alpha=1.8$) & $\bm{0.853} \pm 0.010$ & $\bm{0.145} \pm 0.041$ & /\\
        LogSumExp ($\alpha=5$) & $0.782 \pm 0.012$ & / & $\bm{52.02} \pm 9.16$\\
        Quadratic & $0.804 \pm 0.009$ & / & $\bm{55.10} \pm 11.92$\\
        \bottomrule
        \end{tabular}}
        \vspace{-5pt}
\end{table}

\textbf{Multi-target MNIST Generation.}~~DRE can be used in the sampling-importance-resampling (SIR) paradigm \citep{liu1998sequential,doucet2000sequential}. Suppose we want to obtain samples from $p_1, \ldots, p_{k-1}$ while we have a large set of samples from $p_k$. For each $i \in [k-1]$, we can use the density ratio function $\hat{r}_i$ in conjunction with SIR to approximately sample from the target distribution $p_i$ (Algorithm 1 in \citep{grover2019bias}). For this task, we evaluate if the SIR samples for target distribution $p_i$ contains the correct proportion of classes/properties (10 digit numbers in MNIST) and we use $\frac{1}{k-1}\sum_{i=1}^{k-1} \sum_{j=1}^{10} |h_{ij} - \hat{h}_{ij}|$ as the evaluation metric, where $h_{ij}$ and $\hat{h}_{ij}$ denote the desired proportion and sampled proportion for property $j$ in each target-generation task $i$.

\textbf{Multi-distribution Off-policy Policy Evaluation.}~~Suppose we have $k$ different reinforcement learning policies $p_i(a|s)$, each inducing an occupancy measure \citep{syed2008apprenticeship} (i.e, state-action distribution) $\rho_i(s,a)$. Density ratios allow us to conduct off-policy policy evaluation, which estimates the expected return (sum of reward) of target policies $p_1, \ldots, p_{k-1}$ given trajectories sampled from a source policy $p_k$. In this case, we evaluate the following metric to assess the quality of the learned density ratios ($\tau = \{(s_t, a_t)\}_{t=1}^T$ denotes a sequence of state-action pairs):
\begin{align*}
\small
    \frac{1}{k-1}\sum_{i=1}^{k-1} \left|\mathbb{E}_{p_k(\tau)} \left[\sum_{t=1}^T \hat{r}_i(s_t, a_t) r(s_t, a_t)\right] - \mathbb{E}_{p_i(\tau)} \left[\sum_{t=1}^T r(s_t, a_t)\right]\right|
\end{align*}
We summarized the results for CIFAR-10 OOD detection, multi-target MNIST generation and multi-distribution off-policy policy evaluation in Table~\ref{tab:cifar-ood} (omitted results indicates the corresponding method performs worse than listed methods by a large margin on the specific task). We can see that methods induced by proper scoring rules such as multi-class logistic regression, Brier score and pseudo-spherical score tend to have the best performance on the first two tasks. And surprisingly, methods induced by some simple multivariate convex functions such as the LogSumExp and the quadratic function show excellent performance on the third task. These results demonstrate the advantage of our framework in the sense that it offers us extreme flexibility for designing new objectives for multi-distribution DRE that are more suitable for various downstream applications.

\section{Conclusion}
In this paper, we focus on the generalized problem of efficiently estimating density ratios among multiple distributions. We propose a general framework based on expected Bregmand divergence minimization, where each strictly convex function induces a proper loss for multi-distribution DRE. Furthermore, we rederive the theoretical equivalence between the losses of class probability estimation and density ratio estimation, which justifies the use of any strictly proper scoring rules for multi-distribution DRE. Finally, we demonstrated the effectiveness of our framework on various downstream tasks.

\bibliography{iclr2022_conference}
\bibliographystyle{iclr2022_conference}

\input{appendix}

\end{document}

%% file: math_commands.tex
\usepackage{amsmath,amsfonts,bm}
\usepackage{amsthm}
\usepackage{amssymb}
\newtheorem{theorem}{Theorem}

\newtheorem{definition}{Definition}

\usepackage{thmtools} 
\usepackage{thm-restate}

\newcommand{\breg}{\mathbf{B}}
\newcommand{\divg}{\mathbf{D}}
\newcommand{\R}{\mathbb{R}}
\newcommand{\fopt}{f^\circledast}
\newcommand{\LHS}{\mathrm{LHS}}
\newcommand{\RHS}{\mathrm{RHS}}
\newcommand{\1}{\mathbf{1}}
\newcommand{\I}{\mathbf{I}}
\newcommand{\Ical}{\mathcal{I}}
\newcommand{\mA}{\mathbf{A}}
\newcommand{\reg}{\mathrm{reg}}

\newcommand{\defeq}{\mathrel{\mathop:}=}

\newcommand{\heta}{\hat{\veta}}
\newcommand{\hr}{\hat{r}}
\newcommand{\uL}{\underline{L}}
\def\vpi{{\bm{\pi}}}

\def\eqref#1{Eq.~(\ref{#1})}

\def\1{\bm{1}}

\def\rmL{{\mathbf{L}}}

\def\vmu{{\bm{\mu}}}
\def\vtheta{{\bm{\theta}}}
\def\veta{{\bm{\eta}}}
\def\va{{\bm{a}}}
\def\vb{{\bm{b}}}

\def\vh{{\bm{h}}}

\def\vp{{\bm{p}}}
\def\vq{{\bm{q}}}
\def\vr{{\bm{r}}}
\def\vs{{\bm{s}}}
\def\vt{{\bm{t}}}
\def\vu{{\bm{u}}}
\def\vv{{\bm{v}}}

\def\vx{{\bm{x}}}
\def\vy{{\bm{y}}}

\def\mA{{\bm{A}}}

\def\mH{{\bm{H}}}

\DeclareMathAlphabet{\mathsfit}{\encodingdefault}{\sfdefault}{m}{sl}
\SetMathAlphabet{\mathsfit}{bold}{\encodingdefault}{\sfdefault}{bx}{n}

\def\gL{{\mathcal{L}}}

\def\gS{{\mathcal{S}}}

\def\gX{{\mathcal{X}}}
\def\gY{{\mathcal{Y}}}

\def\sP{{\mathbb{P}}}

\def\sR{{\mathbb{R}}}

\newcommand{\E}{\mathbb{E}}

\newcommand{\KL}{D_{\mathrm{KL}}}

\DeclareMathOperator*{\argmin}{arg\,min}

%% file: appendix.tex
\newpage
\appendix
\section{Proofs}
\subsection{Proofs for Section~\ref{sec:preliminary}}
\classificationBregman*
\begin{proof}
For completeness, we provide the proof here. First, we can check that $\uL(P): \Delta_k \to \sR$ is a concave function. Define $\mathbf{L}(P)$ to be the vector $\left(\ell(1, P), \ldots, \ell(k, P)\right)$. Then the entropy function can be represented as $\uL(P) = L(P,P) = \mathbb{E}_{y \sim P} [\ell(y, P)] = P^\top \mathbf{L}(P)$ and similarly $L(P,Q) = P^\top \mathbf{L}(Q)$. 
For $\lambda \in [0,1]$ and $P, Q \in \Delta_k$, we have:
\begin{align*}
    \uL(\lambda P + (1 - \lambda) Q) &= (\lambda P + (1 - \lambda) Q)^\top \mathbf{L}(\lambda P + (1 - \lambda) Q) \\
    &= \lambda P^\top \mathbf{L}(\lambda P + (1 - \lambda) Q) + (1-\lambda) Q^\top \mathbf{L}(\lambda P + (1 - \lambda) Q) \\
    &\geq \lambda P^\top \mathbf{L}(P) + (1 - \lambda) Q^\top \mathbf{L}(Q) = \lambda \uL(P) + (1 - \lambda) \uL(Q)
\end{align*}
where the inequality is because $\ell$ is proper.
Thus $\uL$ is a concave function.
Next we show that the excess risk is a Bregman divergence with convex function $-\uL$. First, observe that
\begin{align*}
    L(P,Q) = P^\top \rmL(Q) = Q^\top \rmL(Q) + (P-Q)^\top \rmL(Q)
\end{align*}
Because $\ell$ is proper, we have:
\begin{align*}
    0 \leq L(P,Q) - L(P,P) &= Q^\top \rmL(Q) + (P-Q)^\top \rmL(Q) - P^\top \rmL(P) \\
    &= - \uL(P) - (- \uL(Q)) - (P - Q)^\top (- \rmL(Q))
\end{align*}

Rearrange the term we get $- \uL(P) \geq (- \uL(Q)) + (- \rmL(Q))^\top (P - Q)$ and therefore $- \rmL(Q)$ is a subderivative of $- \uL$. When $-\uL$ is differentiable, 
its subdifferential contains exactly one subderivative and $- \rmL(Q) = \nabla(- \uL(Q))$. Therefore, we have $\mathrm{reg}(P, Q) = L(P,Q) - L(P,P) = f(P) - f(Q) - \langle \nabla f(Q), P-Q \rangle = \breg_f(P, Q)$ with $f = - \uL$.
\end{proof}

\subsection{Proofs for Section~\ref{sec:multi-class-dre-method}}
\restateEquivalence*

\begin{proof}[Proof of Proposition~\ref{lem:equivalence-opt}]

We first recall that the optimization problem for multi-distribution DRE is of the form
\begin{align}
    \min_{\hat{\vr}:\gX \to \sR^{k-1}_+} 
    \mathbb{E}_{p_k(x)}\left[\langle \nabla f(\hat{\vr}(x)), \hat{\vr}(x) \rangle - f(\hat{\vr}(x))\right] - \sum_{i \in [k-1]} \mathbb{E}_{p_i(x)} [\partial_i f(\hat{\vr}(x))]
    \label{eq-app:multi-class-dre}
\end{align}
and one can use the Fenchel-dual convex conjugate to represent the $f$-divergence as 
\begin{equation}
    \divg_f(P_1,\cdots, P_{k-1}||P_k)= -\min_{\vs:\gX\rightarrow \R^{k-1}}\left[-\sum_{i\in[k-1]}\E_{p_i(x)}[\vs(x)]_i + \E_{p_k(x)}f^*(\vs(x))\right]
    \label{eq-app:obj-fdivg}
\end{equation}
By first-order optimality condition of convex functions, for any $x\in\gX$ the optimal solution $\overline{\vs}(x)$ for \eqref{eq-app:obj-fdivg} satisfies that
\begin{equation*}
    \forall i \in [k-1], x \in \gX,~~ p_i(x)-\partial_i f^*(\overline{\vs}(x))p_k(x)=0\Longleftrightarrow \frac{p_i(x)}{p_k(x)} = \partial_i f^*(\overline{\vs}(x))
\end{equation*}
Therefore $\overline{\vr}(x) = \nabla f^*(\overline{\vs}(x))$ recovers the true density ratios. 

Now we show that under change of variable $\vs(x) = \nabla f(\hat{\vr}(x))$, one can write the problem in \eqref{eq-app:obj-fdivg} equivalently as the one in \eqref{eq-app:multi-class-dre}. First due to the property of the convex conjugate function ($f^{**} = f$), we have:
$$f^*(\vs(x)) = \min_{\vh(x) \in \sR^{k-1}} \langle \vs(x), \vh(x) \rangle - f(\vh(x))$$
Substituting $\vs(x)$ with $\nabla f(\hat{\vr}(x))$, we have:
\begin{equation}
    f^*(\vs(x)) = \min_{\vh(x) \in \sR^{k-1}} \langle \nabla f(\hat{\vr}(x)), \vh(x) \rangle - f(\vh(x)) \label{eq-app:fenchel-dual}
\end{equation}

Taking derivative w.r.t. $\vh$ and due to the strict convexity of $f$ ($\nabla f(\va) = \nabla f(\vb) \Leftrightarrow \va = \vb$), we know that the minimum of \eqref{eq-app:fenchel-dual} achieves at $\overline{\vh}(x) = \hat{\vr}(x)$. Thus we have:
\begin{equation}
    f^*(\vs(x)) = \langle \nabla f(\hat{\vr}(x)), \hat{\vr}(x) \rangle - f(\hat{\vr}(x)) \label{eq-app:fenchel-dual-optimum}
\end{equation}

Plugging \eqref{eq-app:fenchel-dual-optimum} and $\vs(x) = \nabla f(\hat{\vr}(x))$ back to the optimization problem in \eqref{eq-app:obj-fdivg}, we can get the following equivalent problem by flipping a sign of the objective function without changing the optimal solution:
\[
\min_{\hat{\vr}: \gX \to \sR^{k-1}} \E_{p_k(x)}\left[\langle \nabla f(\hat{\vr}(x)), \hat{\vr}(x)\rangle - f(\hat{\vr}(x)) \right] - \sum_{i\in[k-1]}\E_{p_i(x)} \partial_i f(\hat{\vr}(x)),
\]
which is the same as the one in ~(\ref{eq-app:multi-class-dre}).
\end{proof}

\subsection{Proofs for Section~\ref{sec:theory}}\label{app:proof-sec-4}
\bregidentity*
\begin{proof}[Proof of Lemma~\ref{lem:breg-identity}]
For simplicity of notations we let $u_k=v_k=1$ for arbitrary $u,v\in\R^{k-1}$. We first prove the convexity of $\fopt$ by definition. Given any two points $u,v\in\R^{k-1}$ and $\theta\in[0,1]$, one has
\begin{align*}
& \fopt(\theta \vu+ (1-\theta )\vv)\\
= & \left(\sum_{i\in[k]}\left(\theta u_i+(1-\theta)v_i\right)\right)\cdot f\left(\frac{1}{\sum_{i\in[k]}(\theta u_i+(1-\theta)v_i)}\cdot (\theta \vu+(1-\theta)\vv)\right)\\
= & \left(\theta\sum_{i\in[k]}u_i+(1-\theta)\sum_{i\in[k]}v_i\right)\cdot f\left(\frac{1}{\theta\sum_{i\in[k]}u_i+(1-\theta)\sum_{i\in[k]}v_i}\cdot (\theta \vu+(1-\theta)\vv)\right)\\
\stackrel{(\star)}{\le} & \theta \left(\sum_{i\in[k]}u_i\right)f\left(\frac{1}{\sum_{i\in[k]}u_i}\vu\right)+(1-\theta) \left(\sum_{i\in[k]}v_i\right)f\left(\frac{1}{\sum_{i\in[k]}v_i}\vv\right)\\
= & \theta\fopt(\vu)+(1-\theta)\fopt(\vv).
\end{align*}
Here for inequality~$(\star)$ we use the fact that for any convex function $g:\R^n\rightarrow \R$, the perspective function $h(t,x):\R\times \R^n\rightarrow \R$ defined as $h(t,x)\defeq tg(x/t)$ is a function jointly convex in $(t,x)$.

Now to see the identity holds, note we can write
\begin{align*}
\LHS = & f\left(\frac{1}{\sum_{i\in[k]}u_i}\cdot \vu\right)-f\left(\frac{1}{\sum_{i\in[k]}v_i}\cdot \vv\right)\\
& - \left\langle\nabla f\left(\frac{1}{\sum_{i\in[k]}v_i}\cdot \vv\right),\frac{1}{\sum_{i\in[k]}u_i}\cdot \vu - \frac{1}{\sum_{i\in[k]}v_i}\cdot \vv\right\rangle
\end{align*}
and that
\begin{align*}
\RHS = & f\left(\frac{1}{\sum_{i\in[k]}u_i}\cdot \vu\right)-\frac{\sum_{i\in[k]}v_i}{\sum_{i\in[k]}u_i}\cdot f\left(\frac{1}{\sum_{i\in[k]}v_i}\cdot \vv\right)-\frac{1}{\sum_{i\in[k]}u_i}\left\langle \nabla \fopt\left(\vv\right),\vu-\vv \right\rangle\\
\stackrel{(i)}{=} & f\left(\frac{1}{\sum_{i\in[k]}u_i}\cdot \vu\right)-\frac{\sum_{i\in[k]}v_i}{\sum_{i\in[k]}u_i}\cdot f\left(\frac{1}{\sum_{i\in[k]}v_i}\cdot \vv\right)\\
& -\frac{1}{\sum_{i\in[k]}u_i}\left\langle f\left(\frac{1}{\sum_{i\in[k]}v_i}\cdot \vv\right)\1+\left(\I-\frac{1}{\sum_{i\in[k]}v_i}\1 \vv^\top\right)\nabla f\left(\frac{1}{\sum_{i\in[k]}v_i}\cdot \vv\right),\vu-\vv \right\rangle\\
\stackrel{(ii)}{=} & f\left(\frac{1}{\sum_{i\in[k]}u_i}\cdot \vu\right)-\left[\frac{\sum_{i\in[k]}v_i}{\sum_{i\in[k]}v_i}+\frac{\sum_{i\in[k]}(u_i-v_i)}{\sum_{i\in[k]}u_i}\right]\cdot f\left(\frac{1}{\sum_{i\in[k]}v_i}\cdot \vv\right)\\
& + \frac{1}{\sum_{i\in[k]}u_i}\left\langle \nabla f\left(\frac{1}{\sum_{i\in[k]}v_i}\cdot \vv\right), \left(\I-\frac{1}{\sum_{i\in[k]}v_i}\vv\1^\top\right)(\vu-\vv)\right\rangle\\
= & f\left(\frac{1}{\sum_{i\in[k]}u_i}\cdot \vu\right)- f\left(\frac{1}{\sum_{i\in[k]}v_i}\cdot \vv\right)\\
& + \left\langle \nabla f\left(\frac{1}{\sum_{i\in[k]}v_i}\cdot \vv\right), \frac{1}{\sum_{i\in[k]}u_i}\cdot \vu-\frac{1}{\sum_{i\in[k]}u_i}\cdot \vv-\frac{\sum_{i\in[k]}(u_i-v_i)}{\left(\sum_{i\in[k]}u_i\right)\left(\sum_{i\in[k]}v_i\right)}\cdot \vv\right\rangle\\
= & f\left(\frac{1}{\sum_{i\in[k]}u_i}\cdot \vu\right)- f\left(\frac{1}{\sum_{i\in[k]}v_i}\cdot \vv\right)\\
& + \left\langle \nabla f\left(\frac{1}{\sum_{i\in[k]}v_i}\cdot \vv\right), \frac{1}{\sum_{i\in[k]}u_i}\cdot \vu-\frac{1}{\sum_{i\in[k]}v_i}\cdot \vv\right\rangle,
\end{align*}
where we use $(i)$ the gradient formula that $\nabla \fopt(\vv) = \left(\I-\frac{1}{\sum_{i\in[k]}v_i}\1 \vv^\top\right)\nabla f\left(\frac{1}{\sum_{i\in[k]}v_i}\cdot \vv\right)$ by definition of $\fopt$, and $(ii)$ rearranging terms and that $\langle \mA\vv,\vu\rangle = \langle \vu, \mA^\top \vv\rangle$.

Thus, we have shown that $\LHS=\RHS$ and concludes the proof.
\end{proof}

\corBregIdentity*

\begin{proof}[Proof of Proposition~\ref{cor:breg-identity}]
Given any $x\in\gX$, the equality follows by  applying Lemma~\ref{lem:breg-identity} with $\vu = \frac{1}{\pi_k}\vpi_{[1:k-1]}\circ\vr(x)$ and $\vv = \frac{1}{
\pi_k}\vpi_{[1:k-1]}\circ\hat{\vr}(x)$. To see why this is true, note that we have by definition of $\eta_i(x)$ and $\hat{\eta}_i(x)$ that (here $\circ$ implies element-wise multiplication) 
\[
\veta(x) = \frac{\vpi\circ \vr(x)}{\pi_k+\sum_{j\in[k-1]}\pi_j r_j(x)} = \frac{1}{1+\sum_{i\in[k-1]}u_i}\cdot \vu,~\text{and similarly}~\hat{\veta}(\vx) = \frac{1}{1+\sum_{i\in[k-1]}v_i}\cdot \vv.
\]
Consequently applying Lemma~\ref{lem:breg-identity} implies that
\begin{equation}\label{eq:breg-identity-first}
\breg_f(\veta(x),\hat{\veta}(x)) = \frac{1}{1+\sum_{i\in[k-1]}u_i}\breg_{\fopt}(\vu,\vv)
\end{equation}
Note that given any convex function $\fopt$,  we consider its composition with linear map function as 
\[\fopt_\pi(\vr) = \fopt\left(\frac{1}{\pi_k}\vpi_{[1:k-1]}\circ \vr\right) = \fopt(\vu).\]
We note that linear composition preserves convexity and Bregman divergence equality, i.e. we have
\begin{equation}\label{eq:breg-identity-second}
\begin{aligned}
& \breg_{\fopt}(\vu,\vv)  = \fopt(\vu)-\fopt(\vv) - \langle \nabla \fopt(\vv),\vu-\vv \rangle\\
= & \fopt\left(\frac{1}{\pi_k}\vpi_{1:k-1}\circ\vr\right)-\fopt\left(\frac{1}{\pi_k}\vpi_{1:k-1}\circ\hat{\vr}\right) -\left \langle \nabla \fopt\left(\frac{1}{\pi_k}\vpi_{1:k-1}\circ\hat{\vr}\right),\frac{1}{\pi_k}\vpi_{1:k-1}\circ(\vr-
\hat{\vr}) \right\rangle\\
\stackrel{(\star)}{=} & \fopt_{\pi}(\vr) - \fopt_{\pi}(\hat{\vr})-\langle \nabla \fopt_\pi(\vr),\vr-\hat{\vr} \rangle = \breg_{\fopt_{\pi}}(\vr,\hat{\vr}),
\end{aligned}
\end{equation}
where for equality $(\star)$ we use chain rule for taking derivatives of the linear composite mapping. Combining Equations~(\ref{eq:breg-identity-first}) and~(\ref{eq:breg-identity-second}) and replacing $\vu = \frac{1}{\pi_k}\vpi_{[1:k-1]}\circ\vr$ gives the desired result.
\end{proof}

\thmBregmanDRE*
\begin{proof}[Proof of Theorem~\ref{thm:bregman-dre}]
Given the multi-class classification regret under some proper loss $\ell$ in~\eqref{eq:breg-represent-regret} and Proposition~\ref{cor:breg-identity} we have:
\begin{align*}
\reg(\heta; M, \veta, \ell)  & \defeq \gL_\text{CPE}(\heta; D) - \gL_\text{CPE}(\veta; D) = \mathbb{E}_{M(x)} [\breg_f(\veta(x), \heta(x))]\\
& \stackrel{(i)}{=} \sum_{i\in[k]}\pi_i \E_{p_i(x)}\breg_f(\veta(x),\hat{\veta}(x)) = \E_{p_k(x)}\left( \sum_{i\in[k]}\pi_i \frac{p_i(x)}{p_k(x)}\breg_f(\veta(x),\hat{\veta}(x))\right)\\
& \stackrel{(ii)}{=}  \E_{p_k(x)}\left(\left(\pi_k +\sum_{i\in[k-1]}\pi_i r_i(x)\right)\cdot\breg_f(\veta(x),\hat{\veta}(x))\right)\\
& \stackrel{(iii)}{=} \pi_k \E_{p_k(x)}\breg_{\fopt_\pi}(\vr(x),\hat{\vr}(x)),
\end{align*}
where we use $(i)$ the definition of marginal distribution $M(x) = \sum_{i\in[k]}\pi_i p_i(x)$, $(ii)$ the definition of density ratio that $r_i(x) = p_i(x)/p_k(x)$ $\forall x\in\gX, i \in [k]$, and $(iii)$ Proposition~\ref{cor:breg-identity} with the consistent definitions of $\fopt_\pi$ and $\vr$, $\hat{\vr}$ as stated in the theorem.
\end{proof}

\newcommand{\vset}{\mathcal{V}}

\subsubsection{Information Measure in Multi-class Experiments}\label{app:proof-sec-4-info}
In this section, we show that multi-distribution density ratio estimation can be viewed as estimating the statistical information measure \citep{degroot1962uncertainty} in multi-class experiments, under appropriate choices for the convex function $f$.

We first introduce the following definitions in multi-class experiments.
For $\vp \in \Delta_k$, any proper loss function $\ell:[k]\times\Delta^k\to \R$ induces a generalized entropy:
$$H_\ell(\vp) \defeq \inf_{\vq \in \Delta^k}\sum_{i\in[k]}p_i \ell(i,\vq),$$
which measures the uncertainty of the task.
Given a multi-class experiment $D = (\vpi,P_1,\ldots, P_k)=(M,\veta)$ and the generalized entropy $H_\ell:\Delta^k\to \R$ (which is closed concave), the information measure in a multi-class experiment \citep{degroot1962uncertainty,duchi2018multiclass} is defined as the gap between the prior and posterior generalized entropy:
$$\Ical_{H_\ell}(D) = H_\ell(\vpi) - \E_{M(x)}[H_\ell(\veta(x))].$$
We next introduce the following connections between multi-distribution $f$-divergence, generalized entropy and information measure in multi-class experiments. Specifically, \citet{duchi2018multiclass} proved an equivalence between the gap of prior and posterior Bayes risks and the multi-distribution $f$-divergence induced by a convex function $f$ depending on $\ell$ and the prior $\vpi$, demonstrating the utility of multi-distribution $f$-divergence for experimental design of multi-class classification. 

\begin{theorem}[\citep{duchi2018multiclass}]\label{thm:inform}
Given a proper loss $\ell$, its associated generalized entropy $H_\ell$ and a multi-class distribution $D=(\vpi,P_1,\ldots,P_k)=(M, \veta)$, we can define a closed convex function $f_{\ell, \vpi}: \sR^{k-1}_+ \to \sR \cup \{\pm \infty\}$ as
\begin{equation}
    f_{\ell,\vpi}(\vt)\defeq\sup_{\bm{\nu}\in\Delta^k}\left(H_\ell(\vpi)-\sum_{i\in[k-1]}\pi_i\ell(i,\bm{\nu})t_i-\pi_k \ell(k,\bm{\nu})\right)\label{eq:f-from-H-ell}
\end{equation}
We can then express the information measure of multi-class experiments as the multi-distribution $f$-divergence induced by \eqref{eq:f-from-H-ell}:
\begin{align*}
    \Ical_{H_\ell}(D) &= H(\vpi) - \E_{M(x)} [H_\ell(\veta(x))] = \inf_{\bm{\nu}\in\Delta^k}\sum_{i\in[k]}\pi_i \ell(i,\bm{\nu}) - \inf_{\hat{\veta}}\gL(\hat{\veta};D) \\
    &= \divg_{f_{\ell,\vpi}}(P_{1},\ldots,P_{k-1}
    \|P_k).
\end{align*}
\end{theorem}

Given Theorem~\ref{thm:inform} and Proposition~\ref{lem:equivalence-opt}, we know that multi-distribution density ratio estimation by minimizing expected Bregman divergence (\eqref{eq:multi-class-dre}), induced by the convex function $f_{\ell,\vpi}$ defined in \eqref{eq:f-from-H-ell}, corresponds to estimating the statistical information measure in multi-class classification experiments.

\input{examples}

\subsection{More Experimental Details}\label{app:experiment-details}
We provide more details about the problem setup of each task used in our empirical study.

For the synthetic data experiments, we use $k=5$ multivariate Gaussian distributions with identity covariance matrix and different mean vectors:
\begin{align*}
    \vmu_1 &= (1, 0, 0, \ldots)^d \\
    \vmu_2 &= (-1, 0, 0, \ldots)^d \\
    \vmu_3 &= (0, 1, 0, \ldots)^d \\
    \vmu_4 &= (0, -1, 0, \ldots)^d \\
    \vmu_5 &= (1, 0, 0, \ldots)^d
\end{align*}
We use such design so that the density ratios are almost surely well-defined and the numerical optimization with respect to the canonical density ratio vector $\hat{\vr} = (\hat{r}_1, \ldots, \hat{r}_{k-1})$ is more stable. We use a two-layer Multi-Layer Perceptron (MLP) ($\text{Linear}(d, 32) \to \text{Linear}(32, 32) \to \text{Linear}(32, k-1)$) with ReLU activation function to realize the density ratio model.

For CIFAR-10 OOD detection experiments, we set $k=4$ and we construct each distribution as: $p_1$ - samples labeled \{airplane, automobile, bird\}; $p_2$ - samples labeled \{cat, deer, dog, frog\}; $p_3$ - samples labeled \{horse, ship, truck\} and $p_4$ - a uniform mixture of these distributions. We use a standard convolution neural network in the PyTorch tutorial\footnote{ \url{https://pytorch.org/tutorials/beginner/blitz/cifar10_tutorial.html}} with $k-1$ outputs to realize the density ratio model.

For MNIST multi-target generation experiments, we use $k=6$ and we construct each distribution as: $p_1$ - samples labeled \{0,1\};
$p_2$ - samples labeled \{2,3\};
$p_3$ - samples labeled \{4,5\};
$p_4$ - samples labeled \{6,7\};
$p_5$ - samples labeled \{8,9\};
$p_6$ - a mixture of these distributions. We use a two-layer convolutional neural network (Conv(1,32, 3, 1) $\to$ Conv(32, 64, 3, 1) $\to$ Linear(9216, 128) $\to$ Linear(128, $k-1$)) with ReLU activation function to realize the density ratio model.

For multi-distribution off-policy policy evaluation experiments, we conducted experiments on the Half-Cheetah environment in OpenAI Gym \citep{brockman2016openai}. We use soft actor-critic algorithm \citep{haarnoja2018soft} to obtain five different policies with average expected return of \{3811, 5277, 6444, 7397, 5728\} respectively and we learn density ratios between their induced occupancy measures (state-action distributions). We use a three-layer MLP ($\text{Linear}(17, 256) \to \text{Linear}(256, 256) \to \text{Linear}(256, 256) \to \text{Linear}(256, k-1)$) with ReLU activation function to realize the density ratio model.

%% file: examples.tex
\subsection{Examples of Multi-distribution DRE}\label{app:examples-derivations}
\subsubsection{Multi-class Logistic Regression}\label{app:multi-class-lr-derivations}

From Section~\ref{sec:multi-class-link-function}, we know that there is a one-to-one correspondence between a class probability estimator and a density ratio estimator through the link and the inverse link function: $\hat{\vr} = \Psi_\mathrm{dr} \circ \hat{\veta}$ and $\hat{\veta} = \Psi_\mathrm{dr}^{-1} \circ \hat{\vr}$. When the class prior distribution $\vpi$ is uniform, we have:
\begin{equation}\label{eq-app:uniform-prior-link}
    \hr_i(x) = \frac{\hat{\eta}_i(x)}{\hat{\eta}_k(x)} ~~ \text{and} ~~ \hat{\eta}_i(x) = \frac{\hr_i(x)}{\sum_{j=1}^k \hr_j(x)},~~ \text{for all}~i \in [k], x \in \gX.
\end{equation}
To recover the loss of multi-class logistic regression used in \citep{bickel2008multi}, we choose the following convex function (where we use $\hr_k=1$):
\begin{align}
    f(\hr_1, \ldots, \hr_{k-1}) =& \frac{1}{k} \sum_{i=1}^{k} \hr_i \log\left(\frac{\hr_i}{\sum_{j=1}^k \hr_j}\right) \label{eq-app:multi-class-lr-uniform-f} \\
    \partial_i f(\hr_1, \ldots, \hr_{k-1}) =& \frac{1}{k} \log \left(\frac{\hr_i}{\sum_{j=1}^k \hr_j}\right) \quad\text{for}~i\in[k-1]
\end{align}
Thus the loss in \eqref{eq:multi-class-dre} reduces to:
\begin{align*}
    &\frac{1}{k} \mathbb{E}_{p_k(x)} \left[\sum_{i=1}^{k-1} \hr_i(x) \log \frac{\hr_i(x)}{\sum_{j=1}^k \hr_j(x)} - \sum_{i=1}^{k-1} \hr_i(x) \log \hr_i(x) + \left(\sum_{i=1}^k \hr_i(x)\right) \log \left(\sum_{i=1}^k \hr_i(x)\right)\right] - \\
    &\frac{1}{k} \sum_{i=1}^{k-1} \mathbb{E}_{p_i(x)} \left[\log\left(\frac{\hr_i(x)}{\sum_{j=1}^k \hr_j(x)}\right)\right] \\
    = &\frac{1}{k} \mathbb{E}_{p_k(x)} \left[\hr_k(x) \log \left(\sum_{i=1}^k \hr_i(x)\right) \right] - \frac{1}{k}\sum_{i=1}^{k-1} \mathbb{E}_{p_i(x)} \left[\log\left(\frac{\hr_i(x)}{\sum_{j=1}^k \hr_j(x)}\right)\right] \\
    \stackrel{(i)}{=} & - \left(\frac{1}{k}\sum_{i=1}^k \mathbb{E}_{p_i(x)} [\log(\heta_i(x))] \right)
\end{align*}
where $(i)$ is because $\hr_k(x) = 1,~\forall x \in \gX$ and \eqref{eq-app:uniform-prior-link}.

When the class prior $\vpi$ is not uniform, from Section~\ref{sec:multi-class-link-function}, we know that the link and inverse link connecting density ratio estimators and class probability estimators are:
\begin{equation}
    \hr_i = \frac{\pi_k}{\pi_i}\cdot \frac{\hat{\eta}_i}{\hat{\eta}_k}
    \quad \text{and} \quad
    \hat{\eta}_i = \frac{\pi_i\hr_i}{\sum_{j\in[k]}\pi_j\hr_j},~\text{for all}~i\in[k], x\in\gX.
\end{equation}
In this case, we choose the following convex function (where we use $\hr_k = 1$):
\begin{align}
    f(\hr_1, \ldots, \hr_{k-1}) =& \sum_{i=1}^{k} \pi_i\hr_i \log \pi_i\hr_i  - \left(\sum_{i=1}^{k} \pi_i\hr_i\right) \log\left(\sum_{i=1}^{k} \pi_i\hr_i\right) \label{eq-app:multi-class-lr-general-f} \\
    \partial_i f(\hr_1, \ldots, \hr_{k-1}) =& \pi_i \log \left(\frac{\pi_i\hr_i}{\sum_{j=1}^k \pi_i\hr_j}\right) \quad\text{for}~i\in[k-1]
\end{align}
Note that when $\vpi$ is uniform distribution, \eqref{eq-app:multi-class-lr-general-f} reduces to \eqref{eq-app:multi-class-lr-uniform-f}. 

The loss in \eqref{eq:multi-class-dre} reduces to:
\begin{align*}
    & \mathbb{E}_{p_k(x)} \left[\pi_k \hr_k(x) \log \left(\sum_{i=1}^k \pi_i \hr_i(x)\right) - \pi_k \hr_k(x) \log (\pi_k \hr_k(x))\right] - \sum_{i=1}^{k-1} \mathbb{E}_{p_i(x)} \left[\pi_i\log\left(\frac{\pi_i\hr_i(x)}{\sum_{j=1}^k \pi_j\hr_j(x)}\right)\right] \\
    =& -\left(\sum_{i=1}^k \pi_i \mathbb{E}_{p_i(x)}[\log(\hat{\eta}_i(x))]\right) 
\end{align*}
which corresponds to the multi-class logistic regression loss for the class probability estimators $\heta$.

\textbf{Remark.} Interestingly, we noticed that the multi-distribution $f$-divergence associated with the convex function in \eqref{eq-app:multi-class-lr-uniform-f} is the multi-distribution Jensen-Shannon divergence \citep{garcia2012divergences} (also known as information radius \citep{sibson1969information}) up to a constant of $\log k$:
\begin{align*}
    \divg_f(P_1, \ldots, P_k) &= \E_{p_k(x)} \left[f\left(\frac{p_1(x)}{p_k(x)},\ldots, \frac{p_{k-1}(x)}{p_k(x)}\right)\right]\\
    &= \frac{1}{k} \E_{p_k(x)} \left[ \sum_{i=1}^{k} \frac{p_i(x)}{p_k(x)}\log \left(\frac{p_i(x)}{\sum_{j=1}^k p_k(x)}\right) \right] \\
    &= \frac{1}{k} \sum_{i=1}^k \E_{p_i(x)} \left[\log \left(\frac{p_i(x)}{\frac{1}{k}\sum_{j=1}^k p_j(x)}\right)\right] - \log k\\
    &= \frac{1}{k}\sum_{i=1}^k \KL\left(P_i \| \frac{1}{k}\sum_{j=1}^k P_j\right) - \log k
\end{align*}

\subsubsection{Least-squares Importance Fitting}\label{sec-app:lsif}
When the convex function associated with the Bregman divergence is chosen to be:
\begin{align}
    f(\hr_1, \ldots, \hr_{k-1}) &= \frac{1}{2} \sum_{i=1}^{k-1} (\hr_i - 1)^2 = \frac{1}{2} \|\hat{\vr} - \bm{1}\|^2 \\
    \partial_i f(\hr_1, \ldots, \hr_{k-1}) &= \hr_i - 1 \quad\text{for}~i\in[k-1]
\end{align}
The loss in \eqref{eq:multi-class-dre} reduces to:
\begin{align*}
    & \mathbb{E}_{p_k(x)} \left[\sum_{i=1}^{k-1} (\hr_i^2(x) - \hr_i(x)) - \frac{1}{2} \sum_{i=1}^{k-1}(\hr_i^2(x) - 2 \hr_i(x) + 1) \right] - \sum_{i=1}^{k-1} \mathbb{E}_{p_i(x)} \left[\hr_i(x) - 1\right]\\
    =& \frac{1}{2} \mathbb{E}_{p_k(x)} \left[\sum_{i=1}^{k-1} (\hr_i ^2(x) - 1)\right] - \sum_{i=1}^{k-1} \mathbb{E}_{p_i(x)} \left[\hr_i(x) - 1\right]\\
    =& \frac{1}{2} \sum_{i=1}^{k-1} \mathbb{E}_{p_k(x)} \left[\hr_i^2(x) - 1 - 2 \frac{p_i(x)}{p_k(x)}(\hr_i(x) - 1)\right]\\
    =& \frac{1}{2} \sum_{i=1}^{k-1} \mathbb{E}_{p_k(x)} \left[(\hr_i(x) - r_i(x))^2\right] - C
\end{align*}
where $C = \mathbb{E}_{p_k(x)} \left[\|\vr(x) - 1\|^2\right]$ is a constant w.r.t. $\hat{\vr}$ and the minimizer to the above loss function matches the true density ratios, which strictly generalizes the Least-Squares Importance Fitting (LSIF) \citep{kanamori2009least} method to the multi-distribution case.

\subsubsection{KL Importance Estimation Procedure}\label{app:kliep}
When the convex function associated with the Bregman divergence is chosen to be:
\begin{align}
    f(\hr_1, \ldots, \hr_{k-1}) &= \sum_{i=1}^{k-1} (\hr_i \log \hr_i - \hr_i)= \langle \hat{\vr}, \log(\hat{\vr}) \rangle - \|\hat{\vr}\|_1\\
    \partial_i f(\hr_1, \ldots, \hr_{k-1}) &= \log \hr_i \quad\text{for}~i\in[k-1]
\end{align}
The loss in \eqref{eq:multi-class-dre} reduces to:
\begin{align}
    & \mathbb{E}_{p_k(x)}\left[\sum_{i=1}^{k-1} \hr_i(x) \log \hr_i(x) - \sum_{i=1}^{k-1} (\hr_i(x) \log \hr_i(x) - \hr_i(x))\right] - \sum_{i=1}^{k-1} \mathbb{E}_{p_i(x)}[\log \hr_i(x)]\nonumber\\
    =& \mathbb{E}_{p_k(x)}\left[\sum_{i=1}^{k-1} \hr_i(x)\right] - \sum_{i=1}^{k-1} \mathbb{E}_{p_i(x)}[\log \hr_i(x)]\label{eq-app:KLIEP}
\end{align}
This is equivalent to the following constrained optimization problem:
\begin{align*}
    &\argmin_{\hat{\vr}:\gX \to \sR^{k-1}} \sum_{i=1}^{k-1} \KL(p_i(x) \| \hr_i(x)\cdot p_k(x)) = \sum_{i=1}^{k-1} \mathbb{E}_{p_i(x)}\left[\log \left( \frac{p_i(x)}{\hr_i(x) \cdot p_k(x)}\right)\right]\\
    =& \argmin_{\hat{\vr}:\gX \to \sR^{k-1}} - \sum_{i=1}^{k-1} \mathbb{E}_{p_i(x)} [\log \hr_i(x)]\\
    &\text{s.t.}\quad \mathbb{E}_{p_k(x)}[\hr_i(x)] = 1 ~~\text{and}~~ \hr_i(x) \geq 0,  ~~\text{for all}~i\in[k-1].
\end{align*}
which strictly generalizes the Kullback–Leibler Importance Estimation Procedure (KLIEP) \citep{sugiyama2008direct} to the multi-distribution case.

\subsubsection{Basu's Power Divergence for Robust DRE}\label{app:robust-dre}
For some $\alpha > 1$, we choose the following convex function (the $\alpha$-norm of a vector):
\begin{align}
    f(\hr_1, \ldots, \hr_{k-1}) &= \sum_{i=1}^{k-1} \hr_i^{\alpha} = \|\hat{\vr}\|_{\alpha}^{\alpha}\\
    \partial_i f(\hr_1, \ldots, \hr_{k-1}) &= \alpha \hr_i^{\alpha-1}
\end{align}

The loss in \eqref{eq:multi-class-dre} reduces to:
\begin{align}
    & \mathbb{E}_{p_k(x)}\left[\sum_{i=1}^{k-1} \alpha \hr_i^{\alpha}(x) - \sum_{i=1}^{k-1} \hr_i^{\alpha}(x) \right] - \sum_{i=1}^{k-1} \mathbb{E}_{p_i(x)}\left[ \alpha \hr_i^{\alpha-1}(x) \right]\nonumber\\
    =& \sum_{i=1}^{k-1} \mathbb{E}_{p_k(x)}\left[ (\alpha - 1)\hr_i^{\alpha}(x)\right] - \sum_{i=1}^{k-1} \mathbb{E}_{p_i(x)}\left[ \alpha \hr_i^{\alpha-1}(x) \right]\label{eq-app:robust-loss}
\end{align}

To understand the robustness of this formulation,
for each $i \in [k-1]$, we take the derivative of \eqref{eq-app:robust-loss} w.r.t. the parameters in the density ratio model $\hr_i$ and equate it to zero, and we get the following parameter estimation equation:
\begin{equation}
    \mathbb{E}_{p_k(x)} [\hr_i^{\alpha-1}(x) \nabla \hr_i(x)] - \mathbb{E}_{p_i(x)}[\hr_i^{\alpha - 2}(x) \nabla \hr_i(x)] = \bm{0}\label{eq-app:robust-loss-gradient}
\end{equation}

Now we apply the same analysis to the multi-distribution KLIEP method in \eqref{eq-app:KLIEP} and we get the following equation (for each $i \in [k-1]$):
\begin{equation}
    \mathbb{E}_{p_k(x)}[\nabla \hr_i(x)] - \mathbb{E}_{p_i(x)}[\hr_i^{-1}(x)\nabla \hr_i(x)] = \bm{0}
    \label{eq-app:KLIEP-gradient}
\end{equation}

Comparing \eqref{eq-app:robust-loss-gradient} with \eqref{eq-app:KLIEP-gradient}, we can see that the power divergence DRE method in \eqref{eq-app:robust-loss} is a weighted version of the multi-distribution KLIEP method, where the weight $\hr_i^{\alpha - 1}(x)$ controls the importance of the samples. In some scenario where the outlier samples tend to have density ratios less than one, they will have less influence on the parameter estimation, which generalizes the binary Basu's power divergence robust DRE method \citep{sugiyama2012density} to the multi-distribution case. Another interesting observation is that when $\alpha \to 1$, \eqref{eq-app:robust-loss-gradient} recovers the KLIEP gradient in \eqref{eq-app:KLIEP-gradient}; when $\alpha = 2$, the power divergence DRE in \eqref{eq-app:robust-loss} recovers the multi-distribution LSIF method in Section~\ref{sec-app:lsif}.

\subsubsection{More Examples}
When the convex function is chosen to be the Log-Sum-Exp type function (for $\alpha > 0$):
\begin{align}
    f(\hr_1, \ldots, \hr_{k-1}) &= \alpha \log \left(\sum_{i=1}^{k-1} \exp(\hr_i / \alpha)\right) \\
    \partial_i f(\hr_1, \ldots, \hr_{k-1}) &= \frac{\exp(\hr_i/\alpha)}{\sum_{i=1}^{k-1} \exp(\hr_i/\alpha)}
\end{align}
The loss in \eqref{eq:multi-class-dre} can be written as:
\begin{align}
    & \mathbb{E}_{p_k(x)}\left[ \sum_{i=1}^{k-1} \frac{\hr_i(x) \exp(\hr_i(x)/\alpha)}{\sum_{j=1}^{k-1} \exp(\hr_j(x)/\alpha)} - \alpha \log \left(\sum_{i=1}^{k-1} \exp(\hr_i(x) / \alpha)\right)\right] - \sum_{i=1}^{k-1} \mathbb{E}_{p_i(x)} \left[ \frac{\exp(\hr_i(x)/\alpha)}{\sum_{j=1}^{k-1} \exp(\hr_j(x)/\alpha)} \right]\nonumber
\end{align}

We can similarly derive loss functions induced by other convex functions such as the quadratic function $f(\hr_1, \ldots, \hr_{k-1}) = \hat{\vr}^\top \mH \hat{\vr} + \vq^\top \hat{\vr}$, for some positive definite matrix $\mH \succ 0$.